\documentclass{article}

\usepackage[preprint]{neurips_2025}
\usepackage{amsmath}
\usepackage{amsfonts}
\usepackage{amssymb}
\usepackage{amsthm}
\usepackage{natbib}
\usepackage{graphicx}
\usepackage{subcaption}
\usepackage{xcolor}
\usepackage[utf8]{inputenc}
\usepackage[T1]{fontenc}

\usepackage{todonotes}
\usepackage{hyperref}
\usepackage{cleveref}
\usepackage{multirow}

\newtheorem{theorem}{Theorem}
\newtheorem{definition}{Definition}
\newtheorem{proposition}{Proposition}

\newcommand{\adam}[1]{{\color{blue} Adam: #1}}
\newcommand{\mateusz}[1]{{\color{violet} Mateusz: #1}}
\newcommand{\piotr}[1]{{\color{red} Piotr: #1}}
\newcommand{\mikolaj}[1]{{\color{green} Mikołaj: #1}}
\newcommand{\molko}[1]{{\color{pink} MOlko: #1}}

\renewcommand{\adam}[1]{}
 \renewcommand{\mateusz}[1]{}
 \renewcommand{\piotr}[1]{}
 \renewcommand{\mikolaj}[1]{}
 \renewcommand{\molko}[1]{}

\begin{document}

\title{VARSHAP: Addressing Global Dependency Problems in Explainable AI with Variance-Based Local Feature Attribution}

\author{%
  Mateusz Gajewski\\
  Faculty of Computing and Telecommunications \\
  Poznan University of Technology \\
  Poznan, Poland\\
  IDEAS NCBR\\
  \texttt{mg96272@gmail.com} \\
    \And
  Mikołaj Morzy \\
  Faculty of Computing and Telecommunications \\
  Poznan University of Technology \\
  Poznan, Poland\\
\And
  Adam Karczmarz \\ 
  Faculty of Mathematics, Informatics and Mechanics \\
  University of Warsaw\\
  Warsaw, Poland\\
    \And
    Piotr Sankowski \\ 
  Faculty of Mathematics, Informatics and Mechanics \\
  University of Warsaw\\
  Warsaw, Poland\\
  MIM Solutuions\\
  Research Institute IDEAS\\
}

\date{\today}
\maketitle

\begin{abstract}
    Existing feature attribution methods like SHAP often suffer from global dependence, failing to capture true local model behavior. This paper introduces VARSHAP, a novel model-agnostic local feature attribution method which uses the reduction of prediction variance as the key importance metric of features. Building upon Shapley value framework, VARSHAP satisfies the key Shapley axioms, but, unlike SHAP, is resilient to global data distribution shifts. Experiments on synthetic and real-world datasets demonstrate that VARSHAP outperforms popular methods such as KernelSHAP or LIME, both quantitatively and qualitatively. 
\end{abstract}

\section{Introduction}
\label{sec:introduction}

\paragraph{Feature attribution}
\label{subsec:feature.attribution}

Explainable Artificial Intelligence (XAI) has become increasingly crucial as machine learning models are deployed in high-stakes domains such as healthcare, finance, or criminal justice. In these sensitive areas, the transparency and interpretability of model decisions are paramount for ensuring fairness, accountability, and trust. Among the diverse approaches developed to provide explanations about the inner workings of complex models, feature attribution methods have emerged as particularly valuable tools~\citep{holzinger2020explainable}. These methods aim to quantify the contribution of each input feature to a model's prediction. By providing insights into the decision-making process, feature attribution techniques facilitate model debugging, bias detection, and the overall validation of model behavior against domain knowledge.

Feature attribution methods can be broadly categorized into two main types: global feature importance and local feature importance. Global methods aim to provide a single attribution value for each feature used by the model, offering a high-level overview of which features are, on average, most influential in the model's decision-making process. This aggregated perspective is useful for understanding the general behavior and key drivers of the model. In contrast, local feature attribution methods produce a feature attribution for a specific model output, meaning they explain why the model made a particular prediction for an individual instance. Local attributions can therefore offer insights into which attributes were most significant in making a prediction for a particular instance, revealing the nuances of the model's behavior that might be obscured by global approaches~\citep{guidotti2018survey, murdoch2019definitions}. This instance-specific granularity is particularly valuable in contexts where understanding individual predictions is critical.

\paragraph{Shapley values and global dependency}
\label{subsec:shapley.values.and.global.dependency}

Shapley value theory~\cite{shapley1953value}, originating from cooperative game theory, provides a principled method of allocating payout among players. In~\cite{vstrumbelj2011general} authors suggested that the input features in machine learning models can be treated as players and the model output can be seen as a payout, providing a way to use Shapley value in feature attribution.
\adam{w końcu jak należy pisać cytowania? Tak czy inaczej, raczej zamiast In [cite] author .., lepiej brzmi: Strumbelj and Kononenko \cite{vstrumbelj2011general} suggested... Czy nie ma jakiegoś wymagania, żeby pisać citet albo citep? Choć rzeczywiście, numerowane cytowania pewnie trochę oszczędzają miejsca}
The strong mathematical underpinnings of Shapley values, characterized by desirable axioms like efficiency, symmetry, null-player, and additivity, have made this framework a popular choice for generating feature attributions~\cite{li2024shapley}. Among these, SHAP (SHapley Additive exPlanations)~\cite{lundberg2017unified} is a prominent unified method. SHAP assigns each feature an importance value for a particular instance, quantifying its contribution to shifting the model’s output from a baseline to its current value. A key strength of the SHAP framework is its model-agnostic nature, allowing it to be applied to any machine learning model without requiring access to its internal structure or specific assumptions.

Despite their advantages, SHAP-based methods struggle to fully capture local model behavior due to their inherent global dependence. This issue arises because, when calculating a feature's importance for a specific instance, these methods typically simulate feature absence by sampling from that feature's global distribution across the entire dataset. While this maintains properties like additivity, it can produce explanations that fail to accurately reflect the model's behavior specifically around the point of interest.

This limitation was formalized by~\citep{bilodeau2024impossibility},
\adam{znów: zwykle nie widuję "by [numerek]". To może przechodzi, gdy użyje sie numerowania alfabetycznego, typu "by [ABCD15]", ale tak to nie wiem}
who showed that for piecewise linear models, methods like SHAP and Integrated Gradients~\citep{sundararajan2017axiomatic} can produce arbitrarily different attributions for models that behave identically in a local neighborhood by manipulating the model outside that neighborhood. Consequently, the generated attributions may represent average feature effects rather than the precise local sensitivities that drive predictions for particular inputs, potentially misleading analysis for subgroups (e.g., patients affected differently by factors than the general population).


\subsection{Our contribution}
\label{subsec:original.contribution}





In this paper, we address the outlined shortcomings of features attribution methods.
\begin{itemize}
    \item We introduce an \textbf{instance-centered perturbation mechanism} as an approach to handle out-of-coalition features, focusing exclusively on local model behavior.
    
    \item We theoretically justify the use of variance as the only function satisfying desirable and intuitive axioms for feature attribution in a local context.
    
    \item We introduce \textbf{VARSHAP (VARiance-based SHapley Additive exPlanations)}, a novel framework for generating local feature attributions, combining the above theoretical insights to avoid pitfalls stemming from global dependence.
    
    \item We evaluate VARSHAP against two popular model-agnostic methods -- LIME and SHAP -- across multiple attribution metrics and model architectures, demonstrating its advantages.
\end{itemize}
\adam{nie wiem jaka jest kultura ML-owa, ale ta lista to pewnie mogłoby być cztery osobne paragrafy. Można wtedy usunąć pierwszą linijkę. W ogóle może lepiej Our contribution? Znów, nigdy nie widziałem sformułowania "original contribution" (bo co to jest nieoriginal contribution), ale ja jestem z innej branży jednak xd}

\adam{Fajne też by było pierwsze zdanie wprowadzające w stylu: In this paper, we address the outlined shortcomings of SHAP (zamiast tego co jest).}

\section{VARSHAP: VARiance-based SHapley Additive exPlanations}
\label{sec:varshap}
\adam{Ewentualnie tutaj mogłoby się znaleźć jakieś zdanie wstępu, co się będzie działo w tej sekcji}
\subsection{Variance as the basis for feature attributions}
\label{subsec:variance.as.basis}
Traditional SHAP \cite{lundberg2017unified} 
calculates feature importance by measuring how each feature affects the expected model output. Specifically, for a feature $j$ and a subset of features $S$ not containing $j$, SHAP computes the expected model output $\mathbb{E}[\Omega(x_S, X_{-S})]$ where features in $S$ are fixed to their values from the instance being explained, and out-of-coalition features ($X_{-S}$) are treated as random variables sampled from a global background distribution. The difference in expected output when adding feature $j$ to coalition $S$ represents its marginal contribution, which is then weighted by the Shapley kernel across all possible coalitions.

The core methodological innovation of VARSHAP is to redefine local feature importance by quantifying how much the knowledge of a specific feature's value reduces the uncertainty in the model's output within the local vicinity of the instance being explained. The key insight of VARSHAP is to measure feature importance through the lens of variance reduction. Intuitively, when a feature's value is known and fixed (conceptually, when it is added to a coalition of features whose values are already known for the instance in question), the range of possible model outputs, and thus its variance, should decrease if that feature is influential for the prediction at that specific point. The more a feature contributes to reducing this local output variance, the more important it is considered for that particular prediction. Specifically, for each possible subset of features $S$ that does not contain the feature of interest, $j$, VARSHAP compares two scenarios. First, it considers the variance of the model's output when the features in subset $S$ are fixed to their values from the instance being explained, while all other features (including feature $j$ and those not in $S \cup \{j\}$) are perturbed locally around the instance. Second, it considers the variance of the model's output when the features in $S$ and feature $j$ are fixed to their values from the instance, and only the remaining features (those not in $S \cup \{j\}$) are subjected to local perturbation. The difference between these two variances, which represents the reduction in output variance attributable to fixing feature $j$ given the context of known features $S$, is then weighted according to the standard Shapley kernel. Summing these weighted variance reductions across all possible coalitions $S$ provides the VARSHAP attribution value for feature $j$. This approach anchors the explanation in the local behavior of the model, directly addressing the global dependence issue inherent in methods that rely on a global background distribution for feature perturbation or removal.

The choice of variance as the central metric for quantifying feature contributions in VARSHAP is not arbitrary; rather, it stems directly from a fundamental desideratum for feature attribution methods, which is shift invariance. A feature attribution method is considered shift-invariant if adding a constant value to all outputs of the machine learning model does not change the relative importance assigned to any of its input features. This property is crucial because the absolute scale of model outputs often carries less meaning than the relative contributions of features to those outputs. 
We would also like for our function to return zero, when the variance is zero (output does not change at all) and be symmetric, to not set preference for increasing or decreasing output value. 

\begin{proposition}
We propose that any continuous characteristic function $v(x, \Pi, S)$ used for feature attribution that aims to satisfy shift invariance must fundamentally measure the dispersion of model outputs around their local mean. Specifically, such a characteristic function must take the following general form:

$$ v(x, \Pi, S) = \mathbb{E}_{X_{-S} \sim \Pi(x_{-S}|x_S)}[d(\Omega(x_S, X_{-S}) - \mathbb{E}_{X'_{-S} \sim \Pi(x_{-S}|x_S)}[\Omega(x_S, X'_{-S})])]$$

for some attribution function $d$. We posit that the function $d$ must satisfy the axioms of zero property, sign independence, and additivity to ensure meaningful and consistent attributions. The only function $d$ that satisfies the above axioms, and a simple normalization condition $d(1) = 1$, is $d(x) = x^2$. The formal proof of this proposition is provided in the Appendix~\ref{appendix:attribution.function}.
\end{proposition}

Here, $\Omega(x_S, X_{-S})$ is the model's output when features in coalition $S$ are fixed to their values $x_S$ from the instance being explained, and features $X_{-S}$ (those not in $S$) are perturbed according to a local perturbation distribution $\Pi(x_{-S}|x_S)$ centered around the instance $x$. The inner expectation $\mathbb{E}_{X'_{-S} \sim \Pi(x_{-S}|x_S)}[\Omega(x_S, X'_{-S})]$ represents the local expected model output given that features in $S$ are fixed. 

\subsection{Local perturbation framework}
\label{subsec:local.perturbation.framework}

With the justification for using variance as the core of our characteristic function established, the only remaining component to introduce before presenting the full VARSHAP feature attribution formula is the precise mechanism for local perturbation. Given a machine learning model ${\Omega: \mathcal{X} \subseteq \mathbb{R}^d \rightarrow \mathcal{Y} \subseteq \mathbb{R}}$, we define, for any specific instance $x \in \mathcal{X}$ that we wish to explain, a perturbation function $\Pi: \mathbb{R}^d \rightarrow \Delta(\mathbb{R}^d)$. Here, $\Delta(\mathbb{R}^d)$ represents the space of all probability distributions over $\mathbb{R}^d$. This perturbation function is designed such that for any given point $x$, $\Pi(x)$ yields a probability distribution with the property that if a random variable $X$ is drawn from this distribution, then its expected value is the point itself, i.e., $\mathbb{E}[X] = x$. This ensures that the feature perturbations, and consequently the analysis of the model's behavior, remain focused on the local neighborhood of the instance, thereby maintaining the locality of the explanation. This process can be interpreted as introducing controlled uncertainty specifically about the values of the out-of-coalition features, allowing us to measure the impact of knowing a specific feature's value by observing how it reduces this local uncertainty. 

For a concrete and practical implementation, we propose utilizing a multivariate Gaussian perturbation centered at the instance $x$. Specifically, $\Pi(x)$ is defined as $\mathcal{N}(\mu, \Sigma)$, where the mean $\mu$ is set to the instance itself ($\mu=x$), ensuring the perturbation is centered locally. The covariance matrix $\Sigma$ is chosen to be diagonal, meaning $\Sigma_{ij}=0$ for $i \neq j$. The diagonal elements, representing the variance for each feature's perturbation, are defined as $\Sigma_{ii} = \alpha \cdot \hat{\sigma}_i^2$, where $\hat{\sigma}_i^2$ is the estimated variance of feature $i$ observed in the training dataset, and $\alpha > 0$ is a scaling hyperparameter. This Gaussian perturbation formulation offers several advantages: it aligns well with the theoretical underpinnings (e.g., the additivity property of variance for independent perturbations), naturally accounts for the different scales and typical ranges of various features by using their empirical variances $\hat{\sigma}_i^2$, and is computationally efficient to sample from. Furthermore, the hyperparameter $\alpha$ provides a direct mechanism to control the "locality" of the explanation; smaller values of $\alpha$ lead to tighter perturbations around $x$, focusing on more immediate local behavior, while larger values explore a slightly broader neighborhood. For efficiency, VARSHAP performs local perturbations using the marginal probability distribution of features, but the framework also supports conditional probability distributions, a capacity relevant to the lively academic debate concerning the appropriate use of marginal versus conditional probabilities for simulating feature absence~\cite{janzing2020feature}.

\subsection{VARSHAP feature attribution}
\label{subsec:varshap.feature.attribution}

Having established the foundational principles of using variance reduction for local feature importance and defined the local perturbation framework, we can now fully define the VARSHAP feature attribution formula. For a given machine learning model $\Omega: \mathcal{X} \subseteq \mathbb{R}^d \rightarrow \mathcal{Y} \subseteq \mathbb{R}$, an instance $x \in \mathcal{X}$ to be explained, the set of all feature indices $\mathcal{F} = \{1, \dots, m\}$, and a local perturbation function $\Pi$, the VARSHAP attribution $\Phi_j$ for a feature $j \in \mathcal{F}$ is defined as:

\begin{equation}
\Phi_j(\Omega, \Pi, x) = \sum_{S \subseteq \mathcal{F} \setminus \{j\}} \omega(|S|) \left( \text{Var}_\Omega(S \cup \{j\}) - \text{Var}_\Omega(S) \right)
\end{equation}

where $\omega(|S|) = |S|!(k - |S| - 1)!(k!)^{-1}$ is the Shapley kernel, which assigns a weight to each coalition $S$ based on its size $|S|$, out of $k$ total features. Variance of the machine learning model $\Omega$ is formally defined as:

\begin{equation}
    \text{Var}_\Omega(S) = \mathbb{E}_{X_{-S} \sim \Pi(x_{-S}|x_S)} \left[ \left( \Omega(x_S, X_{-S}) - \mathbb{E}_{X'_{-S} \sim \Pi(x_{-S}|x_S)}[\Omega(x_S, X'_{-S})] \right)^2 \right]
\end{equation}

where $x_S$ denotes the components of the instance $x$ corresponding to features in coalition $S$, and $X_{-S}$ (and $X'_{-S}$ for the inner expectation) denotes the random components for features not in $S$, drawn from the local perturbation distribution $\Pi(x_{-S}|x_S)$ which is derived from $\Pi(x)$ by conditioning on $x_S$ (or simply by taking the marginals if perturbations are independent). The term $\Omega(x_S, X_{-S})$ is the model's output with features in $S$ fixed and others perturbed, and $\mathbb{E}_{X'_{-S} \sim \Pi(x_{-S}|x_S)}[\Omega(x_S, X'_{-S})]$ is the local expected output under these conditions. Thus, $\Phi_j$ measures the weighted average marginal contribution of feature $j$ to the reduction of local model output variance across all possible feature coalitions $S$ not containing $j$.

By defining feature attributions using the Shapley value framework applied to the characteristic function $v(S) = \text{Var}_\Omega(S)$ (the local output variance given coalition $S$), VARSHAP inherently satisfies the three fundamental Shapley value axioms, thereby providing theoretically sound and fair explanations. These axioms are:

\begin{itemize}
    \item \emph{efficiency}: The sum of all VARSHAP attributions, $\sum_{j \in \mathcal{F}} \Phi_j(\Omega, \Pi, x)$, precisely equals the total change in variance from a state where all features are unknown (and thus perturbed according to $\Pi(x)$) to a state where all features are known (fixed to their values in $x$). This total change is $\text{Var}_\Omega(\mathcal{F}) - \text{Var}_{\Omega_f}(\emptyset) = 0 - \text{Var}_{\Omega_f}(\emptyset) = -\text{Var}_{\Omega_f}(\emptyset)$, meaning the sum of attributions equals the negative of the initial total variance under full local perturbation. Thus, the attributions fully account for the overall variance reduction potential.
    \item \emph{symmetry}: If two distinct features, $j$ and $k$, have an identical impact on the model's output variance for all possible coalitions $S \subseteq \mathcal{F} \setminus \{j,k\}$ (i.e., $\text{Var}_{\Omega_f}(S \cup \{j\}) - \text{Var}_{\Omega_f}(S) = \text{Var}_{\Omega_f}(S \cup \{k\}) - \text{Var}_{\Omega_f}(S)$), then their VARSHAP attributions will be equal, i.e., ${\Phi_j = \Phi_k}$. This ensures that features contributing identically to variance changes receive equal attribution.
    \item \emph{null player (dummy)}: If a feature $j$ has no influence on the model's output variance regardless of the other features in any coalition (i.e., $\text{Var}_{\Omega_f}(S \cup \{j\}) = \text{Var}_{\Omega_f}(S)$ for all $S \subseteq \mathcal{F} \setminus \{j\}$), then its VARSHAP attribution will be zero: $\Phi_j = 0$. This guarantees that features that do not affect the local output variance receive no importance.
\end{itemize}

In addition to the additivity axiom of variance, VARSHAP also adheres to the linearity property. This property further guarantees the consistency and interpretability of VARSHAP attributions under specific model structures. We define linearity as follows:

\begin{definition}
    A feature attribution method $\Phi$ is linear if for any model $\Omega$ that is decomposable into a sum of single-feature functions, i.e., $\Omega(x) = \sum_{i \in \mathcal{F}} \Omega_i(x_i)$ (where each ${\Omega_f}_i: \mathbb{R} \rightarrow \mathcal{Y}$ operates only on feature $x_i$), and for any perturbation function $\Pi$ that generates statistically independent feature distributions for the out-of-coalition features, the attribution for any feature $i \in \mathcal{F}$ at an instance $x \in \mathcal{X}$ is given by: $\Phi(\Omega, \Pi, x)_i = \Phi(\Omega_i, \Pi, x_i)$.
\end{definition}

Here, $\Phi(\Omega_i, \Pi, x_i)$ represents the attribution calculated for the simpler model $\Omega_i$ that solely depends on feature $x_i$. Linearity guarantees that when features contribute independently to the overall model output (as in an additive model structure), VARSHAP attributions precisely isolate and quantify each feature's individual contribution to the model's output variance. In such scenarios, the importance score assigned to a feature $i$ by VARSHAP for the full model $\Omega$ will be exactly the total variance that feature $i$ would account for if it were the sole input to its respective sub-model $\Omega_i$. A detailed explanation of why the method is linear can be found in Appendix~\ref{app:linearity}.

The linearity property is important for understanding models with additive structures. For instance, consider a linear regression model $\Omega(x) = \sum_{j \in \mathcal{F}} w_j x_j$. If the features $x_j$ are perturbed independently (e.g., using the proposed Gaussian perturbation with a diagonal covariance matrix), VARSHAP's linearity ensures that the attribution for feature $i$, $\Phi_i(\Omega, \Pi, x)$, will be precisely equal to $-w_i^2 \text{Var}(X_i^{\Pi})$. Here, $w_i$ is the coefficient for feature $i$, and $\text{Var}(X_i^{\Pi})$ is the variance of the perturbed feature $X_i$ as defined by $\Pi(x_i)$. The magnitude of this attribution, $w_i^2 \text{Var}(X_i^{\Pi})$, is exactly the component of the total output variance $\text{Var}_\Omega(\emptyset)$ that is directly attributable to feature $i$'s individual variability and its impact via $w_i$. This aligns perfectly with the intuition that features with larger squared coefficients or those that exhibit greater variability under local perturbation (if influential) should be assigned greater importance in explaining the model's output variance.


\section{Case Studies}
\label{sec:case.studies}

To demonstrate the practical advantages and unique insights offered by VARSHAP, this section presents a comparative analysis with two prominent and widely adopted model-agnostic feature attribution methods: SHAP~\cite{lundberg2017unified} and LIME~\cite{ribeiro2016should}. These methods were specifically selected for comparison due to their established popularity within the XAI community and their core design principle of model-agnosticism. The experimental comparison of VARSHAP with SHAP and LIME is presented in Section~\ref{sec:experimental.results}.

To evaluate VARSHAP and compare its explanatory capabilities, a controlled experiment was constructed using synthetic data designed to mimic a simplified healthcare application. This setup allows for a clear understanding of how different attribution methods perform when the underlying data generating process and model behavior can be precisely defined. The experiment considers a scenario with three distinct sub-populations of patients, labeled A, B, and C. For the training phase, these sub-populations consist of 1000, 5000, and 5000 patient samples, respectively, allowing us to investigate how methods perform with varying group sizes and potentially different underlying relationships. The model utilizes two continuous input features, $X_1$ and $X_2$, which represent hypothetical physiological measurements. These features are used to predict a single continuous output quantity, which can be interpreted as a patient's response to a treatment or a risk assessment score.

Two predictive model types were implemented to assess the attribution methods: Neural Network Models (NNM)---standard feedforward networks with an input layer, three hidden layers (50-70-50 neurons), and an output layer---and Ground-Truth Models (GTM). These are not trained models in the conventional sense but direct programmatic instantiations of the known data-generating functions. GTMs provide crucial insights by eliminating model estimation error typically present in trained models, offering a ``ground truth'' benchmark for explanation fidelity. All input features ($X_1$, $X_2$) were normalized across methods and models to ensure fair comparison.

\subsection{Case Study 1: A subset of data with different input-output characteristics}

For the first case study, the data was generated to create specific conditions for testing attribution consistency. This involved creating two distinct datasets:

\begin{itemize}
    \item Dataset 1: In this dataset, all patient groups were designed to follow the mathematical relationship $Y=X_1+0.2 \cdot X_2$. This means that the outcome variable Y for every patient, regardless of their group, is determined by the values of their features $X_1$ and $X_2$ according to this specific formula.
    \item Dataset 2: This dataset introduced a variation. While patient groups A and B still adhered to the same relationship as in Dataset 1 ($Y=X_1+0.2 \cdot X_2$), patient group C was defined by a different formula: $Y=X_1-0.05 \cdot X_1 \cdot X_2$. This change in the underlying relationship for group C creates a global distribution shift between Dataset 1 and Dataset 2.
\end{itemize}

The key experimental design element is that despite the global difference between the two datasets (due to group C's altered behavior), the local conditions for a specific point of interest were kept constant. Specifically, for a patient located at the center of cluster A, represented by the feature values [0, 0], the prediction mechanism yields the same result in both datasets. This setup allows for the investigation of whether local attribution methods provide consistent explanations for this specific point when the broader data distribution changes.

The results, illustrated in Figure~\ref{fig:attributions.comparison}, demonstrate how SHAP, VARSHAP, and LIME perform when analyzing the specific data point [0,0] in both a neural network and ground-truth models across two distinct datasets. The key observation is that SHAP's attributions for this point varied significantly between Dataset 1 and Dataset 2. In Dataset 1, SHAP assigned positive importance to both features $X_1$ and $X_2$, with $X_1$ receiving a higher attribution value. However, in Dataset 2, SHAP's explanation changed: the attribution for $X_1$ increased, while $X_2$ was assigned a negative attribution value. This occurred despite the fact that the underlying model behavior around the point [0,0] was designed to be identical in both datasets. In contrast to SHAP, both VARSHAP and LIME provided more consistent attribution values for the point [0,0] across the two datasets. This stability in explanations from VARSHAP and LIME aligns with the expectation that if the local predictive behavior of a model at a specific point remains unchanged, the explanations for that point should also remain consistent, even if the global data distribution is altered.

\begin{figure}[h]
\centering
\includegraphics[width=\textwidth]{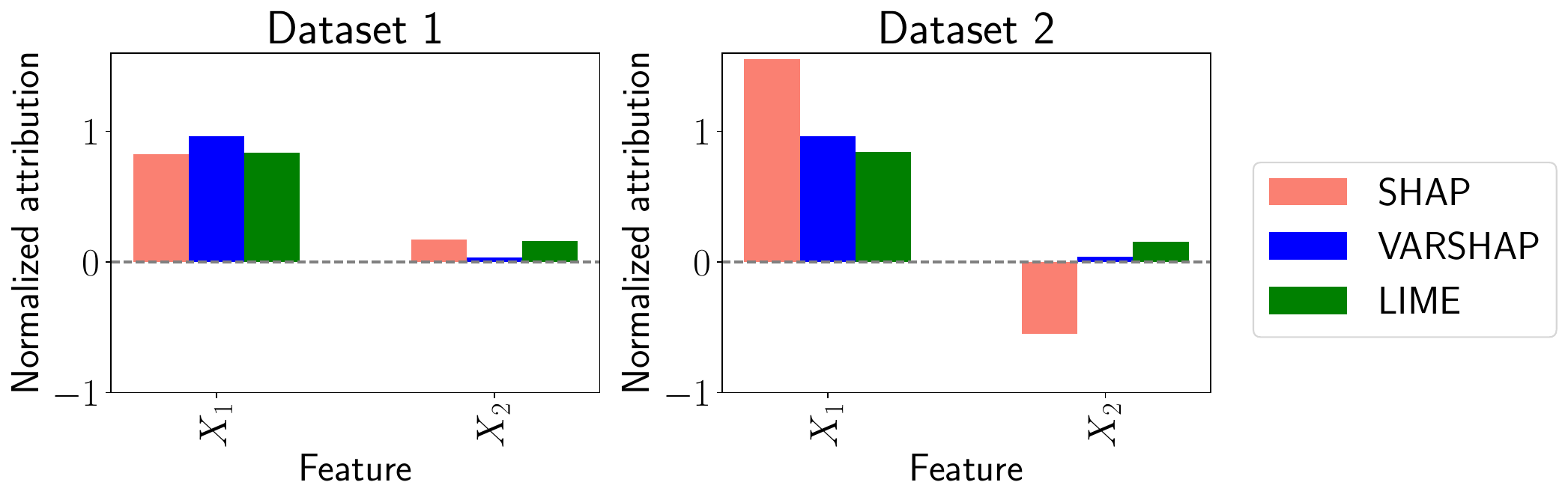}
\caption{Feature attributions for NNMs trained on Dataset 1 (left) and Dataset 2 (right)}
\label{fig:attributions.comparison}
\end{figure}

It is possible that the observed differences in attribution stem from an artifact specific to neural networks rather than being an issue inherent to SHAP. However, by examining the ground-truth models, which can be considered as having complete insight into the data's underlying distribution (akin to an ablation study), we see a strong resemblance in their results (as depicted in Figure~\ref{fig:ground.truth.attributions} ) to those obtained from the NNMs. This close alignment between the GTMs and the NNMs reinforces the argument that the attribution discrepancies are likely due to a fundamental issue within SHAP itself, rather than being a byproduct of the network architecture.

\begin{figure}[ht]
\centering
\includegraphics[width=\textwidth]{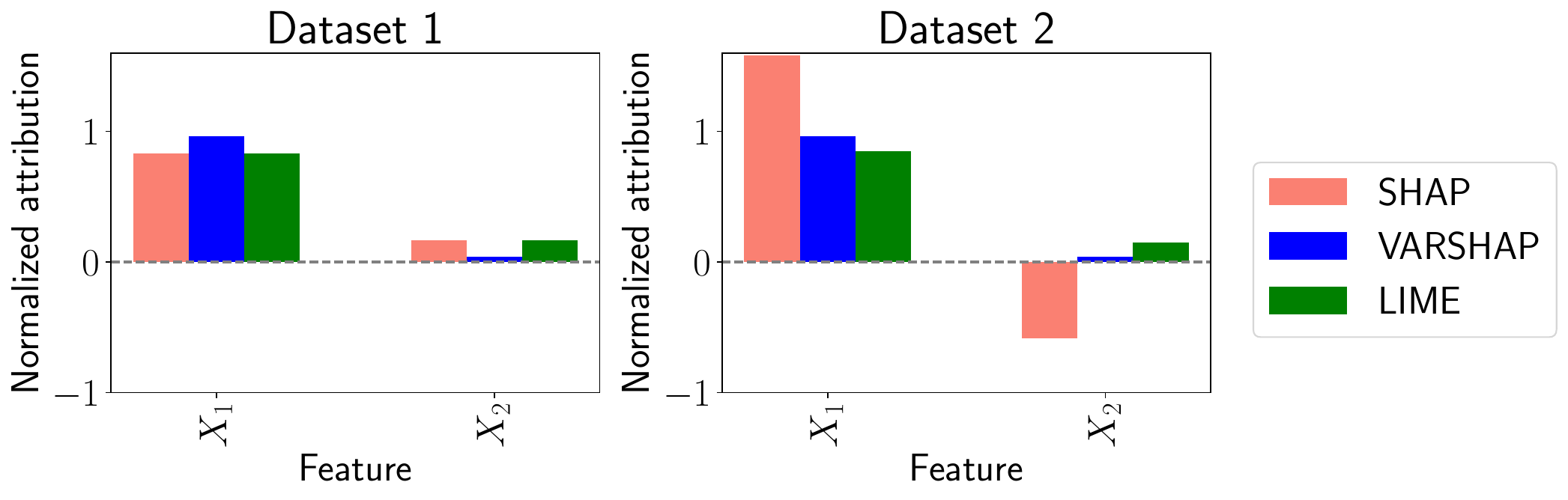}
\caption{Feature attributions for GTMs trained on Dataset 1 (left) and Dataset 2 (right)}
\label{fig:ground.truth.attributions}
\end{figure}

The inherent flaw in SHAP stems from its global nature. Because SHAP considers the overall distribution of the data, changes in one area of the feature space (like in group C) can influence the attributions for instances that are far from that change. This means that SHAP attributions might not accurately represent the specific local factors that are truly driving a particular prediction for an individual instance. In contrast, VARSHAP focuses on the variance created by perturbing features locally around the point being explained. This local sensitivity allows VARSHAP to provide attributions that more faithfully reflect how the model behaves in the immediate vicinity of that specific instance, making it more suitable when the goal is to understand the precise reasons behind individual predictions. The observation that feature attribution patterns in Figure~\ref{fig:ground.truth.attributions} for ground-truth models closely mirror those from neural network models further suggests that this sensitivity to distant changes is a characteristic of SHAP's methodology itself.

\subsection{Case Study 2: Non-linear relationship and irrelevant features}

In the second case study we specifically investigated scenarios involving non-linear relationships and the inclusion of irrelevant features. This experiment was structured to underscore the limitations of explanation methods that depend on linear approximations and to showcase VARSHAP's resilience in more challenging situations. For this purpose, a synthetic Dataset 3 was generated where the target variable is defined by the non-linear equation $Y = |X_1 + X_2|$. The third feature, $X_3$, was introduced which has no correlation with the target variable, and all features ($X_1, X_2, X_3$) were normalized. The choice of the absolute value function introduces a significant challenge for linear explanation techniques due to its distinct non-linearity at the point where $X_1 + X_2 = 0$. Furthermore, according to the null player axiom, the attribution for the irrelevant feature $X_3$ should theoretically be zero. Similar to Case Study 1, both GTMs (exactly following the relationship) and NNMs (architecture [50, 70, 50]) were implemented. This allows for an analysis of both the theoretical behavior of attribution methods and their practical performance on learned models.

When examining the feature attributions for NNMs trained on the dataset characterized by strong non-linearity and the presence of irrelevant features, significant differences emerge between explanation methods, particularly near the non-linear region of the absolute value function. As illustrated in Figure~\ref{fig:non.linear.attributions}, which displays attribution values from SHAP, VARSHAP, and LIME for a representative point in this critical area, LIME assigns a substantial portion of the attribution to the irrelevant feature $X_3$. This incorrect assignment violates the null player axiom, which dictates that features with no influence on the outcome should receive zero attribution. LIME's tendency to attribute importance to $X_3$ is also observed in the GTM where $X_3$ is explicitly constructed to have no influence, confirming that this is an intrinsic issue with LIME's methodology. In contrast, both VARSHAP and SHAP perform correctly in this regard, assigning near-zero attribution to the irrelevant feature $X_3$ and thereby upholding the null player axiom.

\begin{figure}[htbp]
    \centering

        \centering
        \includegraphics[width=\linewidth]{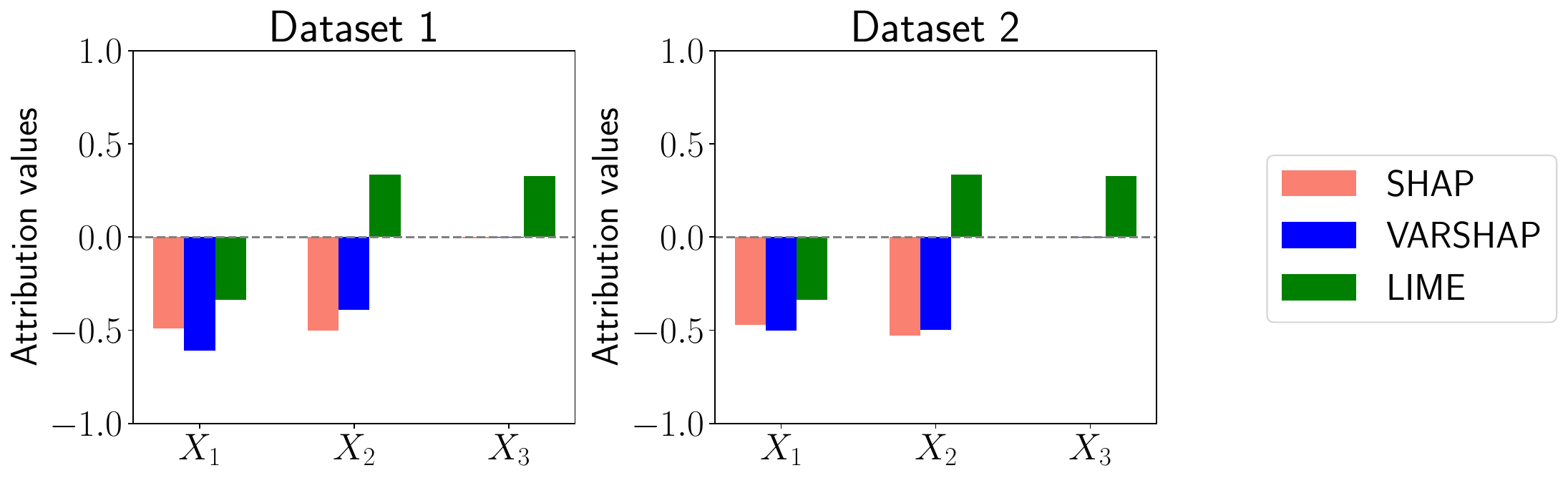}
        \caption{Neural Network Model}
        \label{fig:non.linear.attributions}
        \label{fig:non.linear.attributions.gtm}
        \label{fig:feature.attributions.non.linear} 
\end{figure}

This limitation is further confirmed by Figure~\ref{fig:non.linear.attributions.gtm}, which presents a comparison for the GTM that perfectly implements the absolute value function. Even in this ideal scenario, LIME still incorrectly attributes importance to $X_3$. The core of LIME's issue lies in its fundamental assumption of local linearity, which is ill-suited for handling the sharp non-linearity inherent in the absolute value function. VARSHAP, however, successfully navigates this challenge because its methodology centers on variance reduction. Perturbing the irrelevant feature $X_3$ leads to no change in the output variance, and consequently, VARSHAP correctly assigns it an attribution of zero. To further substantiate the lack of any relationship between $X_3$ and the target variable, partial dependency plots for both models are available in the Appendix~\ref{appendix:pdp}, reinforcing the observation that $X_3$ should indeed receive no attribution.

\section{Experimental Results}
\label{sec:experimental.results}

Evaluating feature attribution is challenging due to the absence of ground truth and the complex nature of "good" explanations. We address this using a LATEC benchmark-inspired framework~\cite{klein2024navigating}, applying diverse faithfulness, robustness, and complexity metrics across various models and datasets. Methods are ranked by their median scores per metric over all combinations, providing a robust assessment to systematically compare VARSHAP with SHAP and LIME. Faithfulness metrics aim to determine if the features highlighted as important by an attribution method genuinely contribute to the model's decision-making process. Robustness metrics, conversely, assess whether features identified as irrelevant are indeed insignificant for the model's predictions. Finally, complexity metrics evaluate the interpretability and conciseness of the explanations provided. For this evaluation, specific metrics were chosen from these categories: Faithfulness (FaithfulnessCorrelation~\cite{bhatt2020evaluating}, FaithfulnessEstimate~\cite{nguyen2020quantitative}and  MonotonicityCorrelation), Robustness (LocalLipschitzEstimate~\cite{alvarez2018towards}, MaxSensitivity~\cite{yeh2019fidelity}, and RelativeInputStability~\cite{agarwal2022rethinking}), and Complexity (Sparseness~\cite{chalasani2020concise}, Complexity, and EffectiveComplexity~\cite{nguyen2020quantitative}). These particular metrics were selected because they are not inherently image-specific and can be readily adapted to scenarios involving a smaller number of features, rendering them suitable for diverse data modalities. The parameters for each metric were determined based on recommendations from their original publications and the LATEC benchmark, and were subsequently fine-tuned to ensure an appropriate distribution of scores. Detailed hyperparameter configurations and the resulting score distributions are provided in the Appendix~\ref{appendix:model.architectures}.

VARSHAP’s performance was evaluated on the Digits dataset~\cite{pedregosa2011scikit} (1,797 8x8 handwritten digit images, 64 features), the Wine Quality dataset~\cite{cortez2009modeling} (tabular wine properties, 11-13 features) and Parkinsons Telemonitoring~\citep{parkinsons_telemonitoring_189}(with a one categorical column removed). For each, three models were trained: a simple NN, a complex NN (CNN for Digits, 5-layer NN for Wine), and a 100-tree random forest (details in Appendix~\ref{appendix:model.architectures}). VARSHAP, tested with varying scaling parameters $\sigma$ (perturbation magnitude) and employing weighted local linear regression as approximation, was compared against model-agnostic methods KernelSHAP~\cite{lundberg2017unified} (using its original sampling of data, a zero baseline, and weighted local linear regression as approximation) and LIME~\cite{ribeiro2016why} (with various sparsity coefficients).  The model-agnostic nature of these selected methods ensures a fair comparison across diverse architectures, as they operate without needing model internals or original training data.

The aggregated rankings (Table~\ref{tab:aggregated.rankings}) show VARSHAP, especially with $\sigma=0.6 $ and $\sigma = 1.0$ and KernelSHAP as the top-performing methods. These approaches, along with other VARSHAP configurations, significantly outperformed LIME variants (ranks 5.51-5.81) and KernelSHAP (0-baseline, 4.86), indicating a clear performance advantage. For detailed insights into performance on specific evaluation metrics, refer to Table~\ref{tab:metric.rankings}

\begin{table}[ht]
\centering
\caption{Aggregated ranking of feature attribution methods across all models and metrics}
\label{tab:aggregated.rankings}
\begin{tabular}{lc}
\hline
\textbf{Method} & \textbf{Average Ranking} \\
\hline
VARSHAP ($\sigma = 0.6$) & \textbf{3.39} \\
VARSHAP ($\sigma = 1.0$) & \textbf{3.39} \\
VARSHAP ($\sigma = 0.3$) & 3.64 \\
KernelShap (data sampling) & 3.74 \\
KernelShap (baseline 0) & 4.86 \\
LIME (sparsity = 1.5) & 5.51 \\
LIME (sparsity = 5.0) & 5.67 \\
LIME (sparsity = 0.5) & 5.81 \\
\hline
\end{tabular}
\end{table}

Table~\ref{tab:metric.rankings}  presents the detailed rankings of methods by individual metrics. The methods show distinct performance patterns. For faithfulness, KernelShap with data sampling often ranks top (e.g., FaithfulnessCorrelation), while VARSHAP is competitive, leading in MonotonicityCorrelation. VARSHAP variants excel in robustness metrics, which indicates VARSHAP’s variance-based approach yields explanations more stable to slight input perturbations. Furthermore, VARSHAP leads in complexity metrics, suggesting it produces more concise and interpretable explanations.
\begin{table}[ht]
\centering
\caption{Rankings by individual metrics (lower is better)}
\label{tab:metric.rankings}
\resizebox{\textwidth}{!}{%
\begin{tabular}{l|cccccc|ccc|ccc}
\hline
\multirow{2}{*}{\textbf{Method}} & \multicolumn{6}{c|}{\textbf{Faithfulness}} & \multicolumn{3}{c|}{\textbf{Robustness}} & \multicolumn{3}{c}{\textbf{Complexity}} \\
\cline{2-13}
 & FC & FC\_B & FE & FE\_B & MC & MC\_B & LLE & MS & RIS & SP & CP & ECP \\
\hline
VARSHAP ($\sigma$=0.6$\sigma$ = 0.6
$\sigma$=0.6) & 2.75 & 5.38 & 2.63 & 4.50 & 3.13 & 4.75 & \textbf{2.00} & 2.71 & 5.86 & 2.22 & \textbf{1.89} & 3.00 \\
VARSHAP ($\sigma$=0.3$\sigma$ = 0.3
$\sigma$=0.3) & 2.50 & 4.75 & 2.50 & 5.13 & 3.38 & 5.50 & 2.29 & \textbf{1.71} & 6.86 & 2.00 & 2.11 & 5.00 \\
VARSHAP ($\sigma$=1.0$\sigma$ = 1.0
$\sigma$=1.0) & 2.63 & 5.25 & 2.75 & 5.75 & 2.50 & 5.25 & 2.14 & 3.29 & 5.86 & \textbf{1.78} & 2.00 & 2.11 \\
KernelShap (data) & \textbf{2.13} & 6.63 & \textbf{2.13} & 6.63 & \textbf{2.00} & 5.38 & 4.14 & 2.29 & 5.43 & 4.11 & 4.11 & 4.44 \\
KernelShap (0) & 5.63 & 3.75 & 5.63 & 3.88 & 6.50 & 4.00 & 4.43 & 5.00 & \textbf{2.86} & 5.00 & 4.89 & 3.22 \\
LIME (1.5) & 6.38 & 3.25 & 6.88 & \textbf{3.00} & 6.50 & \textbf{3.13} & 7.43 & 7.00 & 2.71 & 7.11 & 7.22 & 4.78 \\
LIME (0.5) & 7.50 & 3.88 & 6.88 & 3.38 & 5.63 & 4.00 & 6.71 & 7.57 & 3.00 & 7.44 & 7.22 & 6.11 \\
LIME (5.0) & 6.50 & \textbf{3.13} & 6.63 & 3.75 & 6.38 & 4.00 & 6.86 & 6.43 & 3.43 & 6.33 & 6.56 & 7.33 \\
\hline
\end{tabular}}
\small{FC: FaithfulnessCorrelation, FC\_B: FaithfulnessCorrelation\_black, FE: FaithfulnessEstimate, FE\_B: FaithfulnessEstimate\_black, MC: MonotonicityCorrelation, MC\_B: MonotonicityCorrelation\_black, LLE: LocalLipschitzEstimate, MS: MaxSensitivity, RIS: RelativeInputStability, SP: Sparseness, CP: Complexity, ECP: EffectiveComplexity}
\end{table}

\paragraph{Conclusions}
In conclusion, comprehensive experiments demonstrate VARSHAP’s advantages as a feature attribution method. It performs comparably or often superior to established methods like KernelSHAP and LIME, with particular excellence in explanation robustness and complexity. VARSHAP also exhibits stability across its different $$\sigma$$ parameter values, indicating that while precise tuning can be beneficial, it is not critical for achieving strong performance. These findings validate VARSHAP’s theoretical underpinnings, showing its innovative, axiomatically-derived variance-based characteristic function translates directly into practical performance advantages. Crucially, VARSHAP retains SHAP’s desirable properties without being prone to errors from global dependency, outperforms LIME, and provides explanations grounded in Shapley’s axiomatic system.
\adam{może zrobić z tego paragraf Conclusion, żeby bardziej rzucał się w oczy?}

\section{Limitations and Future Work}
\label{sec:limitations}

The current implementation of VARSHAP uses independent Gaussian perturbations as its local sampling mechanism. While effective in our experiments, this represents a simplification that may not optimally capture complex feature interactions. The perturbation mechanism could be extended beyond independent Gaussian distributions to conditional probability distributions or other distribution families that better represent the underlying data structure.

This extension represents a promising direction for future work, particularly for improving attribution of correlated features. Additionally, more sophisticated perturbation mechanisms could help mitigate vulnerabilities to adversarial attacks on explainability, such as those demonstrated by~\citep{slack2020fooling}, which exploit the inclusion of out-of-distribution points in generating explanations. By constraining perturbations VARSHAP could potentially provide explanations less prone to such attacks.

\bibliographystyle{abbrv}
\bibliography{references}

\appendix

\section{Appendix}

\subsection{Shift invariance and centered distributions}
\label{appendix:shift.invariance}

We begin by establishing that a function of a random variable satisfying shift invariance must be expressible in terms of centered random variables.

\begin{proposition}\label{prop:shift.invariance}
Let $X$ be a random variable representing model outputs, and let $d$ be a distribution functional satisfying shift invariance, i.e., $d(X+c) = d(X)$ for all constants $c$. Then $d$ depends only on the distribution of the centered random variable $X - \mathbb{E}[X]$.
\end{proposition}

\begin{proof}
By shift invariance, for any constant $c$:
\begin{equation}
d(X + c) = d(X)
\end{equation}

In particular, for $c = -\mathbb{E}[X]$:
\begin{equation}
d(X - \mathbb{E}[X]) = d(X)
\end{equation}

This shows that $d$ depends only on the distribution of the centered random variable and that without loss of generality, we may assume $\mathbb{E}[X] = 0$ when studying properties of $d$.
\end{proof}

\subsection{Determination of the attribution function $d$}
\label{appendix:attribution.function}

In the context of the SHAP framework, we consider expectations over features outside the coalition. Our characteristic function takes the form:

\begin{equation}
v(x, \Pi, S) = \mathbb{E}[d(\Omega(x_S, X_{-S}) - \mathbb{E}[\Omega(x_S, X_{-S})])]
\end{equation}

Where $d$ is a continuous function we seek to determine. By~\Cref{prop:shift.invariance}, we can simplify by assuming $\mathbb{E}[\Omega(x_S, X_{-S})] = 0$ without loss of generality, giving us:

\begin{equation}
v(x, \Pi, S) = \mathbb{E}[d(\Omega(x_S, X_{-S}))]
\end{equation}

We propose three axioms that the attribution function $d$ should satisfy:

\begin{itemize}
    \item \emph{zero property}: $d(0) = 0$. This axiom dictates that if perturbing features locally around the instance $x$ (while features in $S$ are fixed) does not lead to any change in the model’s output whatsoever (i.e., the deviation from the local expected value is zero), then these perturbed features should collectively receive zero attribution in that specific context.
    \item \emph{sign independence}: $\forall x \in \mathbb{R} : d(-x) = d(x)$. This property ensures that positive and negative deviations from the local expected value are treated symmetrically. The magnitude of the deviation is what matters for attribution, not its direction. This means the method does not inherently prefer features that increase the model's output over those that decrease it, or vice-versa, when assessing their contribution to output variability.
    \item \emph{additivity}: For independent random variables $A$ and $B$, representing independent contributions to the overall deviation from the expected value, we require $\mathbb{E}[d(A + B)] = \mathbb{E}[d(A)] + \mathbb{E}[d(B)]$. This property is essential for ensuring that our attribution method behaves additively when combining the effects of independent sources of variation. If the total deviation is a sum of independent parts, their attributed importance should also sum up accordingly.
\end{itemize}

\begin{theorem}
The only function $d$ satisfying Axioms 1-3 and normalization $d(1) = 1$ is $d(x) = x^2$.
\end{theorem}

\begin{proof}

To simplify, we may assume $\mathbb{E}[X] = \mathbb{E}[Y] = 0$ for any random variables considered. Let's also assume bounded random variables to ensure existence of all expected values.

By Sign Independence (Axiom 2), $d$ must be an even function. Let's assume $d$ is continuous.

The Additivity axiom requires that for independent random variables $X$ and $Y$ with zero means:
\begin{equation}
\mathbb{E}[d(X+Y)] = \mathbb{E}[d(X)] + \mathbb{E}[d(Y)]
\end{equation}

By linearity of expectation, this is equivalent to:
\begin{equation}
\mathbb{E}[d(X+Y) - d(X) - d(Y)] = 0
\end{equation}

Consider independent random variables $X$ and $Y$ with symmetric two-point distributions:
\begin{equation}
P(X=s) = P(X=-s) = \frac{1}{2}
\end{equation}
\begin{equation}
P(Y=t) = P(Y=-t) = \frac{1}{2}
\end{equation}

For these distributions:
\begin{equation}
\mathbb{E}[d(X)] = \frac{d(s) + d(-s)}{2} = d(s)
\end{equation}
\begin{equation}
\mathbb{E}[d(Y)] = d(t)
\end{equation}
\begin{equation}
\mathbb{E}[d(X+Y)] = \frac{d(s+t) + d(s-t) + d(-s+t) + d(-s-t)}{4}
\end{equation}

Since $d$ is even, $d(-s+t) = d(s-t)$ and $d(-s-t) = d(s+t)$. Thus:
\begin{equation}
\mathbb{E}[d(X+Y)] = \frac{d(s+t) + d(s-t)}{2}
\end{equation}

The additivity condition becomes:
\begin{equation}
\frac{d(s+t) + d(s-t)}{2} - d(s) - d(t) = 0
\end{equation}

When $s = t = 0$, we get $d(0) = 0$, confirming Axiom 1.

Setting $s = kt$ for integer $k$, we can show by induction that:
\begin{equation}
d(kt) = k^2 d(t)
\end{equation}

\textbf{Induction Hypothesis}: Assume $f(kt) = k^2f(t)$ for some integer $k \geq 1$.

\textbf{Goal}: Prove $f((k+1)t) = (k+1)^2f(t)$

From our functional equation with $s = kt$:
$$\frac{f(kt+t) + f(kt-t)}{2} - f(kt) - f(t) = 0$$

Rearranging to isolate $f((k+1)t)$:
$$f((k+1)t) = 2f(kt) + 2f(t) - f(kt-t)$$

Applying the induction hypothesis:
\begin{align*}
& f(kt) = k^2f(t) \\
& f(kt-t) = f((k-1)t) = (k-1)^2f(t) \text{ (for $k > 1$)} \\
\end{align*}

For $k = 1$:
\begin{align*}
f(2t) &= 2f(t) + 2f(t) - f(0) \\
&= 4f(t) \\
&= 2^2f(t)
\end{align*}

For $k > 1$:
\begin{align*}
f((k+1)t) &= 2k^2f(t) + 2f(t) - (k-1)^2f(t) \\
&= 2k^2f(t) + 2f(t) - (k^2 - 2k + 1)f(t) \\
&= 2k^2f(t) + 2f(t) - k^2f(t) + 2kf(t) - f(t) \\
&= k^2f(t) + 2kf(t) + f(t) \\
&= (k^2 + 2k + 1)f(t) \\
&= (k+1)^2f(t)
\end{align*}

This proves the induction step, confirming that $f(kt) = k^2f(t)$ for all positive integers $k$.

For rational numbers $a/b$:
\begin{equation}
d\left(\frac{a}{b}\right) = \frac{a^2}{b^2}d(1)
\end{equation}
Because:
\begin{align}
f(t) &= \frac{f(kt)}{k^2} \\
\\
\text{Setting } t &= \frac{1}{b} \text{ and } k = b: \\
\\
f\left(\frac{1}{b}\right) &= \frac{f\left(b \cdot \frac{1}{b}\right)}{b^2} \\
&= \frac{f(1)}{b^2}
\end{align}
Since the function is continuous, this extends to all real numbers:
\begin{equation}
d(x) = x^2 d(1)
\end{equation}

With the normalization condition $d(1) = 1$, we get:
\begin{equation}
d(x) = x^2
\end{equation}

This proves that the squared function is the unique solution in this case.

Now we show that for $d(x) = x^2$ these axioms hold for all random variables, concluding that variance is unique function (up to normalization constant):

\begin{enumerate}
\item \textbf{Zero Property:} $d(0) = 0^2 = 0$. This axiom is trivially satisfied.

\item \textbf{Sign Independence:} $d(-x) = (-x)^2 = x^2 = d(x)$. This confirms that $d$ treats positive and negative deviations equally.

\item \textbf{Additivity:} For independent random variables $X$ and $Y$:
\begin{align}
\text{Var}(X + Y) &= \mathbb{E}[((X + Y) - \mathbb{E}[X + Y])^2] \\
&= \mathbb{E}[((X + Y) - (\mu_X + \mu_Y))^2] \\
&= \mathbb{E}[((X - \mu_X) + (Y - \mu_Y))^2] \\
&= \mathbb{E}[(X - \mu_X)^2 + 2(X - \mu_X)(Y - \mu_Y) + (Y - \mu_Y)^2] \\
&= \mathbb{E}[(X - \mu_X)^2] + 2\mathbb{E}[(X - \mu_X)(Y - \mu_Y)] + \mathbb{E}[(Y - \mu_Y)^2] \\
&= \text{Var}(X) + 2\mathbb{E}[(X - \mu_X)(Y - \mu_Y)] + \text{Var}(Y) \\
&= \text{Var}(X) + 2\mathbb{E}[X - \mu_X]\mathbb{E}[Y - \mu_Y] + \text{Var}(Y) \quad \text{(by independence)} \\
&= \text{Var}(X) + 2 \cdot 0 \cdot 0 + \text{Var}(Y) \\
&= \text{Var}(X) + \text{Var}(Y)
\end{align}
\end{enumerate}
\end{proof}

The variance characteristic function emerges naturally from fundamental axioms about feature attribution. When we combine this result with the Shapley value framework, we obtain VARSHAP, a method that maintains the axiomatic properties of Shapley values while focusing on the local behavior of the model through variance reduction.

\subsection{VARSHAP Linearity}
\label{app:linearity}

For a model $\Omega$ that is decomposable into a sum of single-feature functions, i.e., $\Omega(x) = \sum_{i\in F} \Omega_i(x_i)$ (where each $\Omega_i: \mathbb{R} \to Y$ operates only on feature $x_i$), and for a perturbation function $\Pi$ that generates statistically independent feature distributions for the out-of-coalition features, the attribution for any feature $j \in F$ at an instance $x \in X$ is given by: 

\begin{equation}
\Phi(\Omega, \Pi, x)_j = \Phi(\Omega_j, \Pi, x_j) = \textrm{Var}_{\Pi}(\Omega_j(X_j))
\end{equation}

Where $\textrm{Var}_{\Pi}(\Omega_j(X_j))$ represents the variance of the output of function $\Omega_j$ when feature $j$ is perturbed according to the perturbation distribution $\Pi(x)$.

Recall that the VARSHAP attribution for feature $j$ is defined as:

\begin{equation}
\Phi_j(\Omega, \Pi, x) = \sum_{S\subseteq F\setminus\{j\}} \omega(|S|) (\textrm{Var}_\Omega(S) - \textrm{Var}_\Omega(S \cup \{j\}))
\end{equation}

Where $\omega(|S|) = |S|!(|F|-|S|-1)!/|F|!$ is the Shapley kernel, and $\textrm{Var}_\Omega(S)$ represents the variance of the model's output when features in $S$ are fixed to their values from instance $x$ and features outside $S$ are perturbed according to $\Pi$.

For an additive model $\Omega(x) = \sum_{i\in F} \Omega_i(x_i)$, when features are perturbed independently (as specified by our perturbation function $\Pi$ with diagonal covariance matrix), we can express the variance terms as follows:

\begin{equation}
\textrm{Var}_\Omega(S) = \mathbb{E}_{X_{-S}\sim\Pi(x_{-S}|x_S)}\left[\left(\Omega(x_S, X_{-S}) - \mathbb{E}[\Omega(x_S, X_{-S})]\right)^2\right]
\end{equation}

Due to the additive structure of $\Omega$ and the independence of features, this can be rewritten as:

\begin{equation}
\textrm{Var}_\Omega(S) = \mathbb{E}\left[\left(\sum_{i\in S} \Omega_i(x_i) + \sum_{l\in F\setminus S} \Omega_l(X_l) - \mathbb{E}\left[\sum_{i\in S} \Omega_i(x_i) + \sum_{l\in F\setminus S} \Omega_l(X_l)\right]\right)^2\right]
\end{equation}

Since the features in $S$ are fixed, their contribution to the expected value is constant. Therefore:

\begin{equation}
\textrm{Var}_\Omega(S) = \mathbb{E}\left[\left(\sum_{l\in F\setminus S} \Omega_l(X_l) - \mathbb{E}\left[\sum_{l\in F\setminus S} \Omega_l(X_l)\right]\right)^2\right]
\end{equation}

For independent random variables, the variance of a sum equals the sum of the variances:

\begin{equation}
\textrm{Var}_\Omega(S) = \sum_{l\in F\setminus S} \textrm{Var}(\Omega_l(X_l))
\end{equation}

Similarly, for $S \cup \{j\}$:

\begin{equation}
\textrm{Var}_\Omega(S \cup \{j\}) = \sum_{l\in F\setminus(S\cup\{j\})} \textrm{Var}(\Omega_l(X_l))
\end{equation}

Therefore, the difference in variance is:

\begin{equation}
\begin{aligned}
\textrm{Var}_\Omega(S) - \textrm{Var}_\Omega(S \cup \{j\}) &= \sum_{l\in F\setminus S} \textrm{Var}(\Omega_l(X_l)) - \sum_{l\in F\setminus(S\cup\{j\})} \textrm{Var}(\Omega_l(X_l)) \\
&= \textrm{Var}(\Omega_j(X_j))
\end{aligned}
\end{equation}

Substituting this into the VARSHAP attribution formula:

\begin{equation}
\Phi_j(\Omega, \Pi, x) = \sum_{S\subseteq F\setminus\{j\}} \omega(|S|) \cdot \textrm{Var}(\Omega_j(X_j))
\end{equation}

Since $\sum_{S\subseteq F\setminus\{j\}} \omega(|S|) = 1$ (a property of the Shapley kernel) and  $\textrm{Var}(\Omega_j(X_j)$ does not depend on $S$, we get:

\begin{equation}
\Phi_j(\Omega, \Pi, x) = \textrm{Var}(\Omega_j(X_j))
\end{equation}

For the specific case of linear regression models where $\Omega(x) = \sum_{j\in F} w_j x_j$, the attribution for feature $j$ becomes:

\begin{equation}
\Phi_j(\Omega, \Pi, x) = w_j^2 \cdot \textrm{Var}(X_j)
\end{equation}

Where $\textrm{Var}(X_j)$ is the variance of feature $j$ under the perturbation distribution $\Pi$. This confirms that when features contribute independently to the model output, VARSHAP attributions precisely isolate and quantify each feature's individual contribution to the model's output variance.

\subsection{Partial dependency plots}
\label{appendix:pdp}
\paragraph{Neural Network Model}

\begin{figure}[htbp]
    \centering 

    \begin{subfigure}[b]{0.48\textwidth} 
    \centering 
        \includegraphics[width=\linewidth]{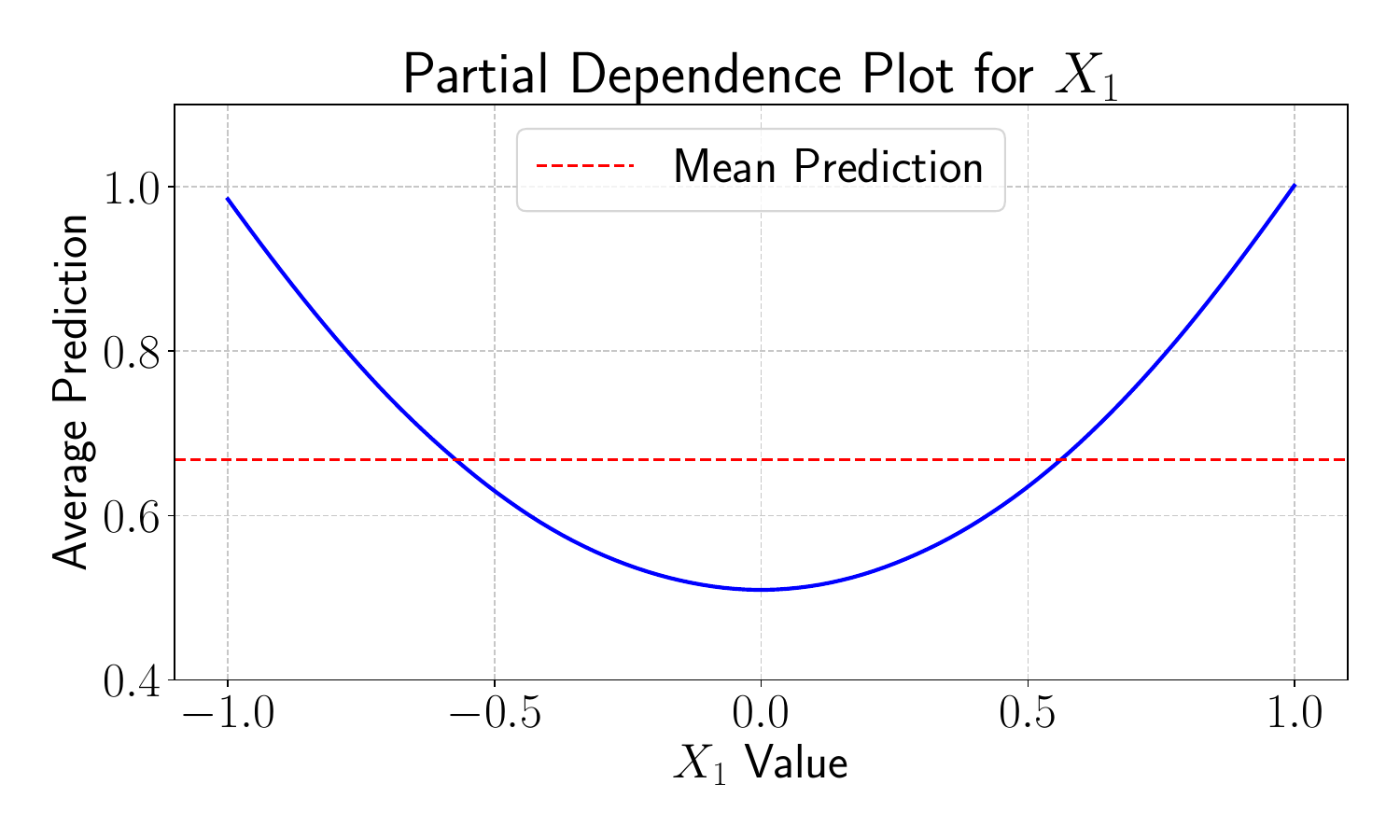}
    \caption{Partial dependency plot for feature $X_1$ for Neural Network Model.} 
                                                        
    \end{subfigure}
    \hfill 
    \begin{subfigure}[b]{0.48\textwidth} 
\centering 
        \includegraphics[width=\linewidth]{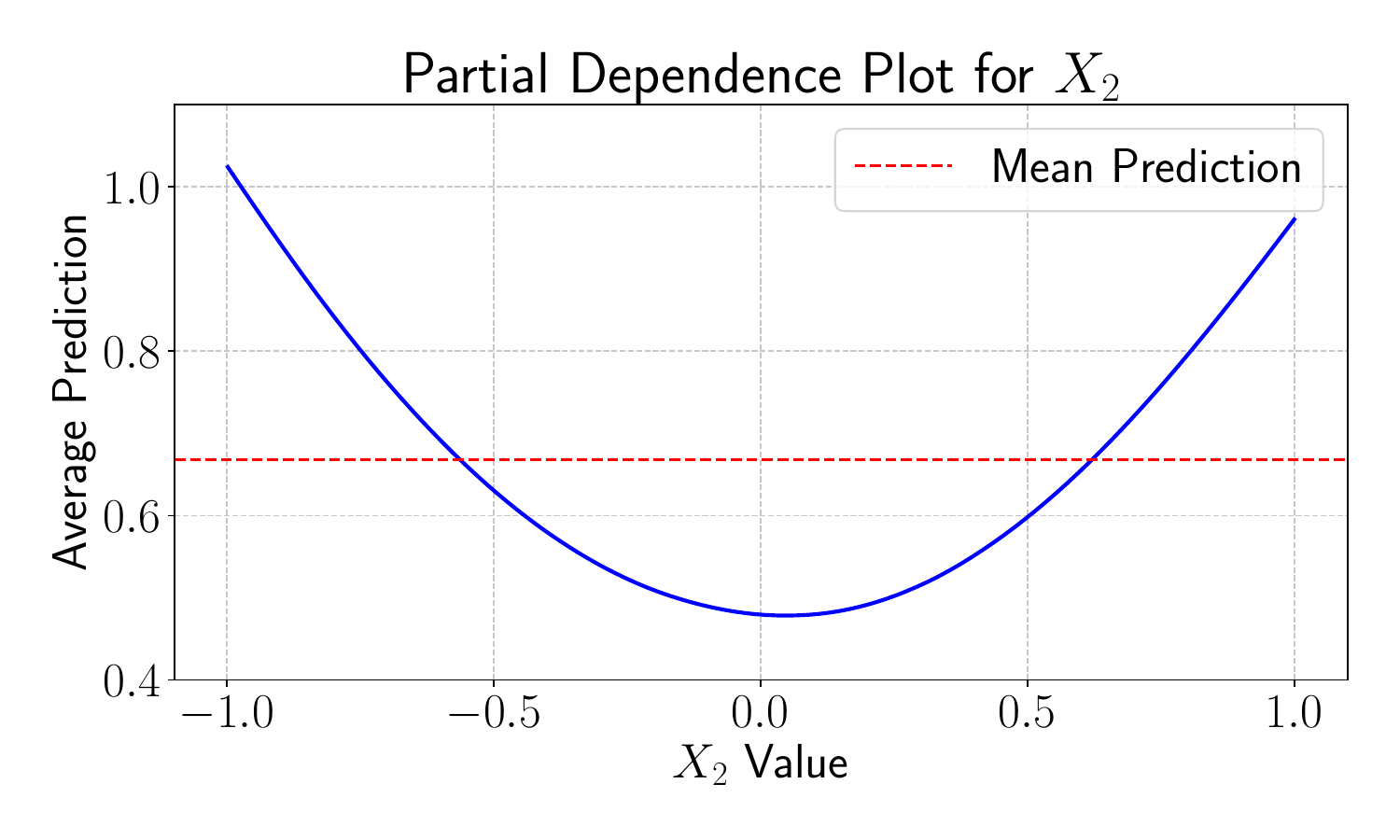}
    \caption{Partial dependency plot for feature $X_2$ for Neural Network Model.} 
                                                                  
    \end{subfigure}

    \caption{Feature attributions trained on Dataset 3 \molko{13.05 the text in those figures is too small}} 
\end{figure}
\begin{figure}[htbp]
    \centering 

    \begin{subfigure}[b]{0.48\textwidth} 
    \centering 
        \includegraphics[width=\linewidth]{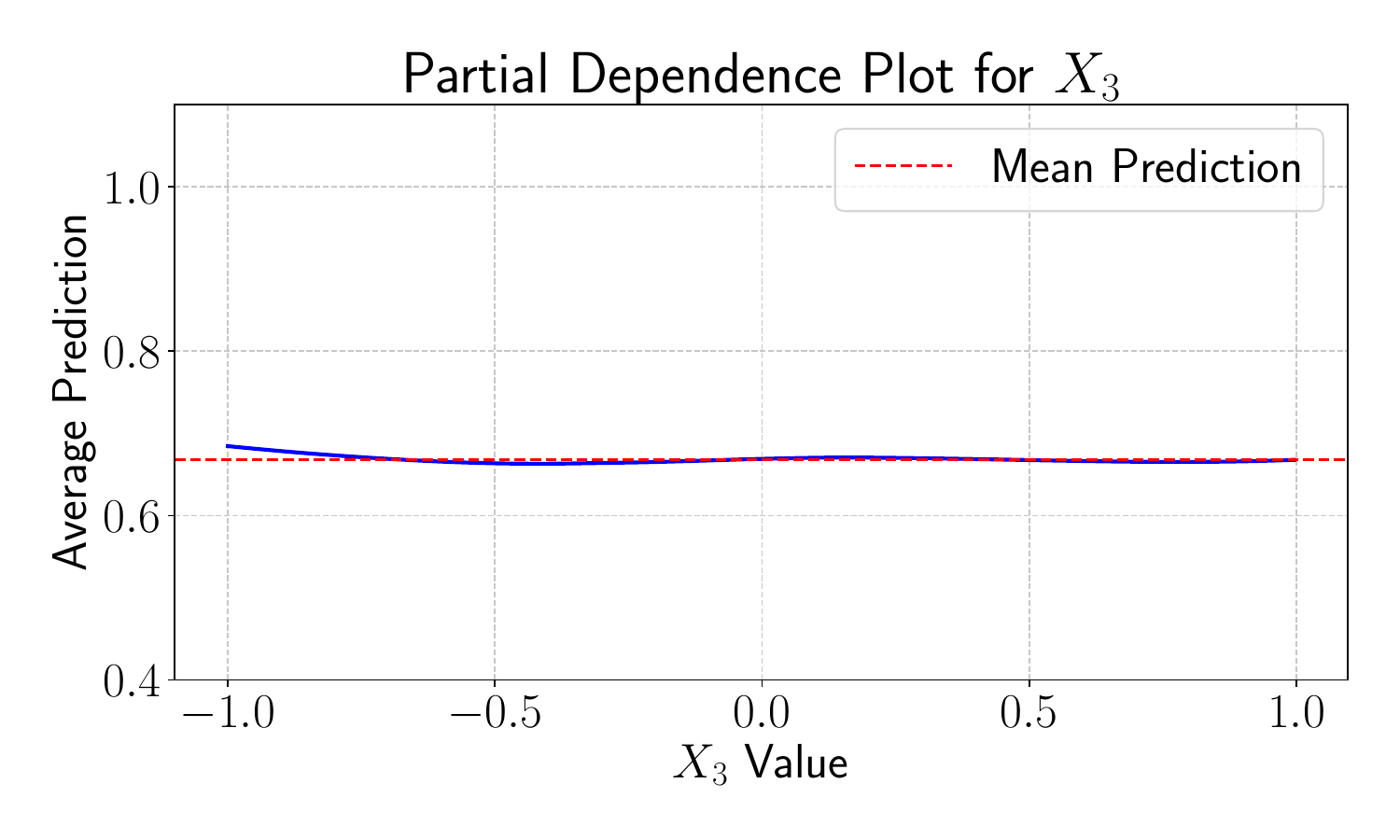}
    \caption{Partial dependency plot for feature $X_3$ for Neural Network Model.} 
                                                        
    \end{subfigure}
    \hfill 
    \begin{subfigure}[b]{0.48\textwidth} 
\centering                                                     
    \end{subfigure}

\end{figure}

\paragraph{Ground-truth Model}
\begin{figure}[htbp]
    \centering 

    \begin{subfigure}[b]{0.48\textwidth} 
    \centering 
        \includegraphics[width=\linewidth]{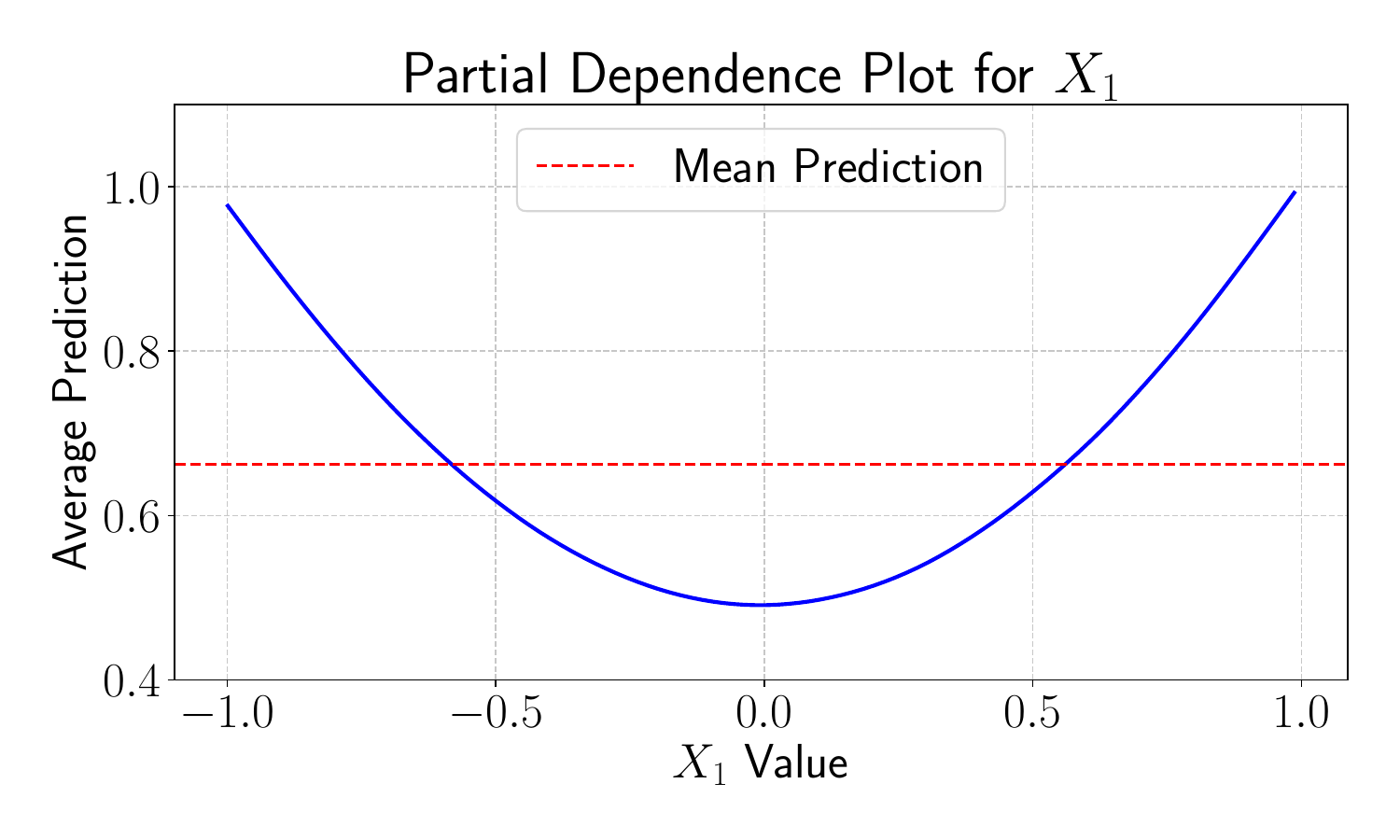}
    \caption{Partial dependency plot for feature $X_1$ for Ground-truth Model.} 
                                                        
    \end{subfigure}
    \hfill 
    \begin{subfigure}[b]{0.48\textwidth} 
\centering 
        \includegraphics[width=\linewidth]{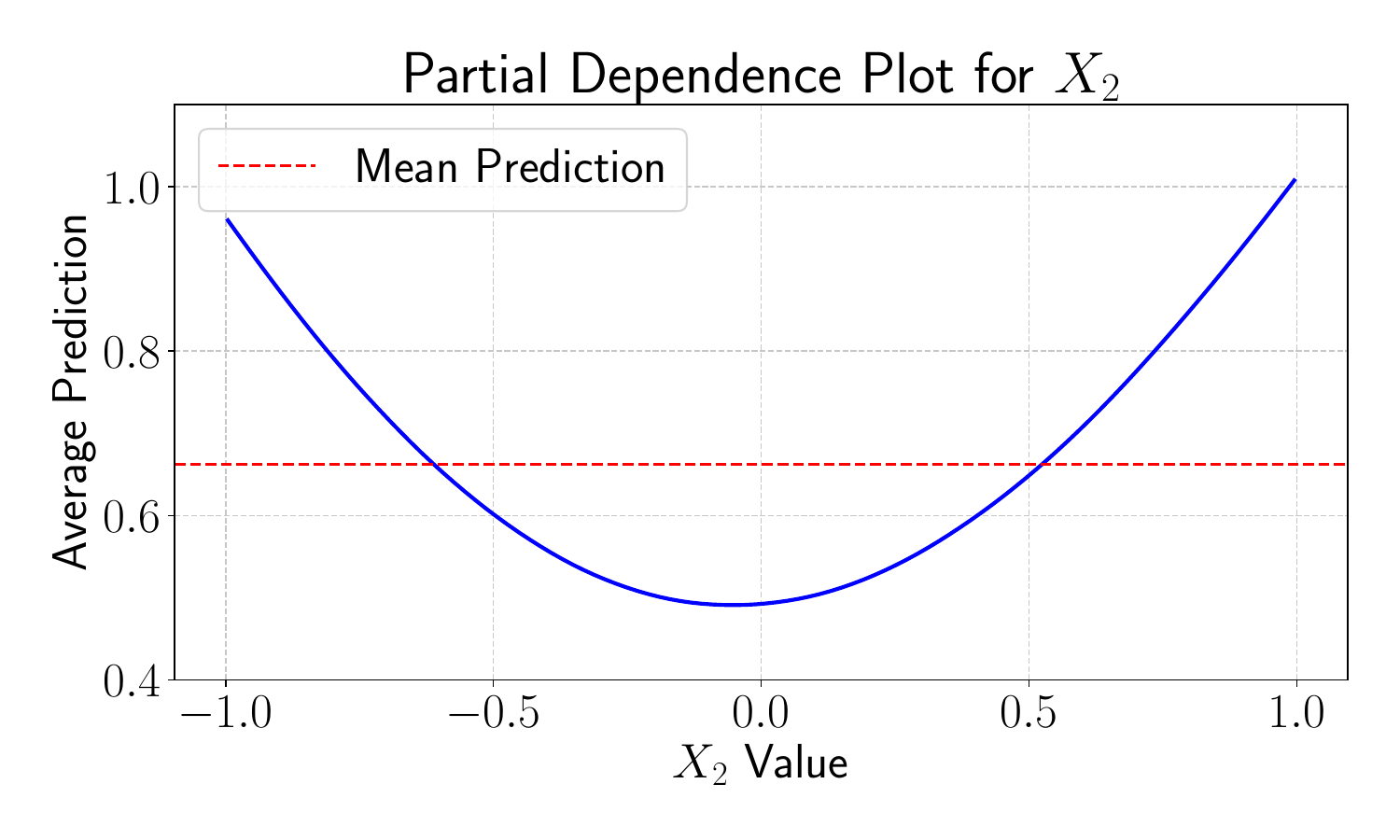}
    \caption{Partial dependency plot for feature $X_2$ for Ground-truth Model.} 
                                                                  
    \end{subfigure}

\end{figure}
\begin{figure}[htbp]
    \centering 

    \begin{subfigure}[b]{0.48\textwidth} 
    \centering 
        \includegraphics[width=\linewidth]{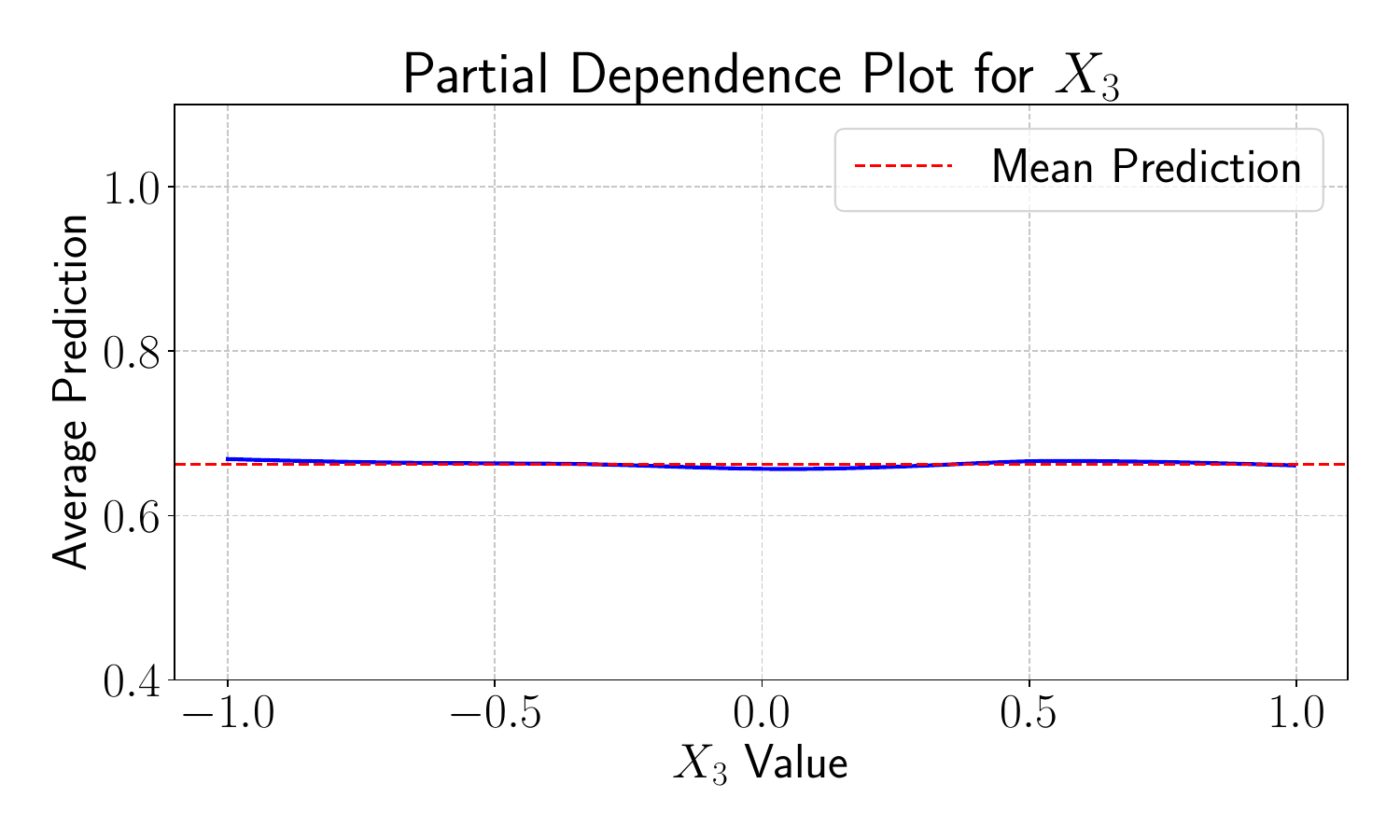}
    \caption{Partial dependency plot for feature $X_3$ for Ground-truth Model.} 
                                                        
    \end{subfigure}
    \hfill 
    \begin{subfigure}[b]{0.48\textwidth} 
\centering                                                     
    \end{subfigure}

\end{figure}

\subsection{Model architectures and hyperparameters}
\label{appendix:model.architectures}
In all cases $80\%$ of dataset was used as a training dataset and $20\%$ as a test dataset. 
For neural network optimization, we performed grid search hyperparameter tuning over batch sizes (32, 64, 128) and learning rates (0.1, 0.01, 0.001) to identify the configuration yielding optimal model performance. All neural network models used Adam~\citep{kingma2014adam} as an optimizer.

\paragraph{Digits Dataset Models:}
\begin{itemize}
    \item \textbf{DigitsNNModel:} A simple multi-layer perceptron (MLP) with two hidden layers, each with 128 neurons. The architecture consists of an input layer accepting 64-dimensional feature vectors, followed by two ReLU-activated hidden layers with dropout regularization (rate=0.2), and a final output layer with 10 classes corresponding to digits 0-9. Chosen hyperparameters: \textbf{lr=0.01, batch size = 256}
    
    \item \textbf{DigitsConvNNModel:} A convolutional neural network designed for the 8×8 digits images. The architecture includes a single 2D convolutional layer with 32 filters and 3×3 kernel size, followed by a flattening operation and two fully-connected layers. The hidden layer has 64 neurons with ReLU activation and dropout (rate=0.2), while the output layer has 10 neurons corresponding to the digit classes. Chosen hyperparameters: \textbf{lr=0.01, batch size = 256}
    
    \item \textbf{DigitsTreeModel:} RandomForestClassifier with n\_estimators=20, max\_depth=4, min\_samples\_split=2.
\end{itemize}
\paragraph{Parkinson Dataset Models:}
\begin{itemize}
    \item \textbf{ParkinsonNNModel:} A standard MLP with two hidden layers, each containing 128 neurons. The architecture begins with an input layer accepting 20-dimensional feature vectors, followed by two ReLU-activated hidden layers with dropout regularization (rate=0.2), and a final single-neuron output layer for regression tasks. Chosen hyperparameters: \textbf{lr=0.001, batch size = 64}
    
    \item \textbf{ParkinsonDeeperNNModel:} A deeper neural network consisting of five layers designed for regression tasks. The architecture begins with an input layer accepting 20-dimensional feature vectors, followed by three hidden layers of 96 neurons each and a fourth hidden layer with 48 neurons, all using ReLU activation and dropout regularization (rate=0.1). The model concludes with a single-neuron output layer for predicting the continuous target variable. Chosen hyperparameters: \textbf{lr=0.001, batch size = 32}
    
    \item \textbf{ParkinsonTreeModel:} RandomForestRegressor with n\_estimators=20, max\_depth=4, min\_samples\_split=2.
\end{itemize}

\paragraph{Wine Dataset Models:}
\begin{itemize}
    \item \textbf{WineNNModel:} A standard MLP with two hidden layers, each containing 128 neurons. The architecture begins with an input layer accepting 11-dimensional feature vectors (representing wine characteristics), followed by two ReLU-activated hidden layers with dropout regularization (rate=0.2), and a final output layer with 10 neurons. Chosen hyperparameters: \textbf{lr=0.01, batch size = 32}
    
    \item \textbf{WineDeeperNNModel:} A deeper network with five layers designed for the wine quality prediction task. The architecture consists of an input layer accepting 11-dimensional feature vectors, followed by three hidden layers with 96 neurons each and a fourth hidden layer with 48 neurons. All hidden layers use ReLU activation and dropout regularization (rate=0.1). The output layer contains 10 neurons corresponding to wine quality scores. Chosen hyperparameters: \textbf{lr=0.001, batch size = 256}
    
    \item \textbf{WineTreeModel:} RandomForestClassifier with n\_estimators=100, max\_depth=None, min\_samples\_split=2.
\end{itemize}

\subsection{Metrics hyperparameters and histograms}
\label{appendix:metrics_params}
For our evaluation, we configured metrics with the following hyperparameters:

\textbf{Faithfulness Metrics}: FaithfulnessCorrelation was configured with 100 runs and subset size 6, using either black'' or uniform'' perturbation baselines. FaithfulnessEstimate used single-feature steps with both perturbation baseline types. MonotonicityCorrelation used single-feature steps with 10 samples.

\textbf{Robustness Metrics}: LocalLipschitzEstimate was implemented with 10 samples, perturbation standard deviation of 0.2, and zero mean. MaxSensitivity used 10 samples with a lower bound of 0.02. RelativeInputStability was configured with 10 samples.

\textbf{Complexity Metrics}: Sparseness and Complexity were used with default parameters, while EffectiveComplexity employed an epsilon value of 0.05.

All metrics were applied consistently across models and datasets to ensure fair comparison between attribution methods. Below (Figure~\ref{fig:faithfulnesscorrelation}) we present a histogram for the FaithfulnessCorrelation metric on neural network models for the Parkinson dataset. Additional histograms for other metrics, methods, models and datasets can be found in the supplementary materials.
\label{appendix:metric_hist}
\begin{figure}[htbp]
    \centering 

        \includegraphics[width=\linewidth]{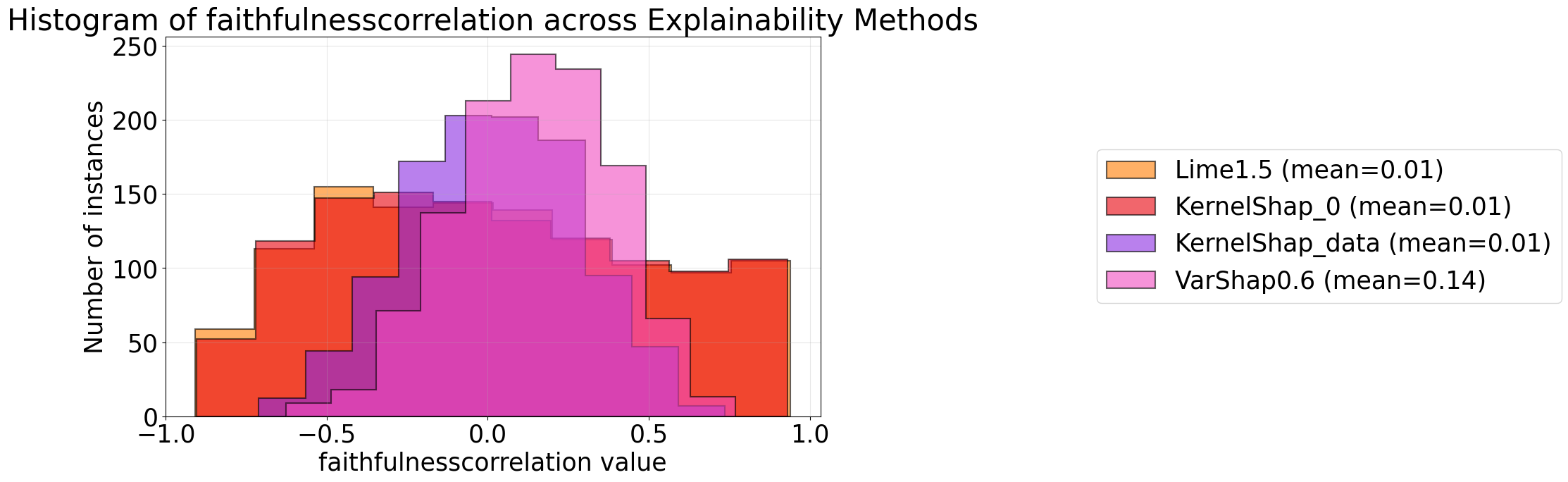}

    \caption{Histogram for metric faithfulnesscorrelation for NN moel for Parkinson dataset} 
    \label{fig:faithfulnesscorrelation} 
\end{figure}

\subsection{Decision surfaces}
\label{appendix:decision.surfaces}

Figure~\ref{fig:decision.surfaces} provides a visual representation of the model's decision-making process for both Dataset 1 and Dataset 2, while Figure~\ref{fig:decision.surfaces.ground.truth} depicts decision surface for Dataset 3. Imagine these "decision surfaces" as topographical maps where the contours indicate how the model arrives at a particular prediction. The specific point under examination, [0,0], is marked in red on these maps. What becomes evident from these visualizations is that even though the overall "landscape" or global behavior of the model changes quite dramatically in Dataset 2 (especially in the area corresponding to group C, which has a different underlying rule), the "local terrain" or gradient immediately surrounding the point [0,0] remains largely the same in both datasets. This means that for data points very close to [0,0], the model makes decisions in a similar fashion in both scenarios. It's worth noting a subtle difference: for the neural network models, the learning process itself introduces very slight alterations to the decision surface even in this local region. In contrast, for the ground-truth models, which are perfectly defined, the decision surfaces in the immediate vicinity of [0,0] are exactly identical. The fact that similar patterns in how importance is attributed to features were seen in both the more complex neural network models and the simpler, perfectly known ground-truth models is significant. It confirms that the way these attribution methods behave (e.g., SHAP's changing attributions versus VARSHAP and LIME's consistency) is a fundamental characteristic of the methods themselves, rather than some quirk or artifact introduced by the model training process.

\begin{figure}[ht]
\centering
\includegraphics[width=\textwidth]{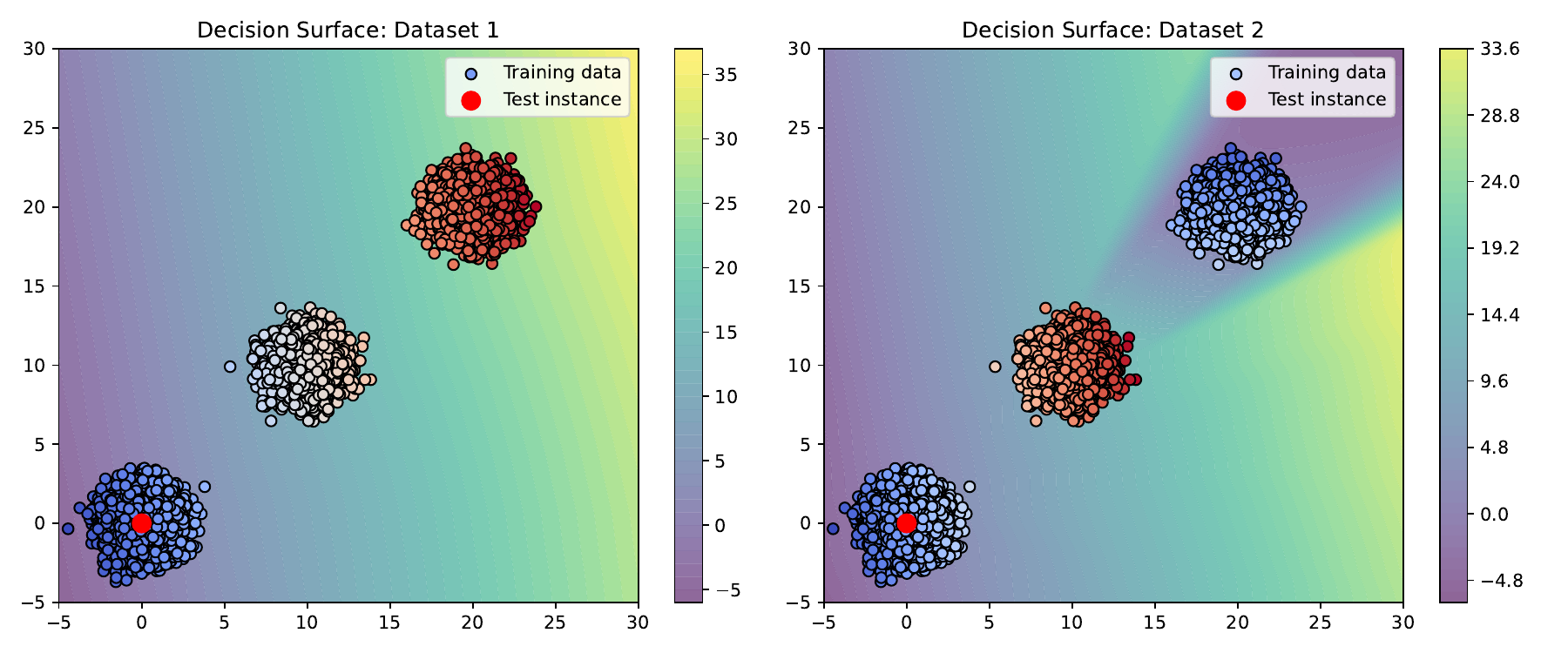}
\caption{Decision surface visualization of NNMs trained on Dataset 1 (left) and Dataset 2 (right). The test instance at $[0,0]$ is marked in red.}
\label{fig:decision.surfaces}
\end{figure}

\begin{figure}[ht]
\centering
\includegraphics[width=\textwidth]{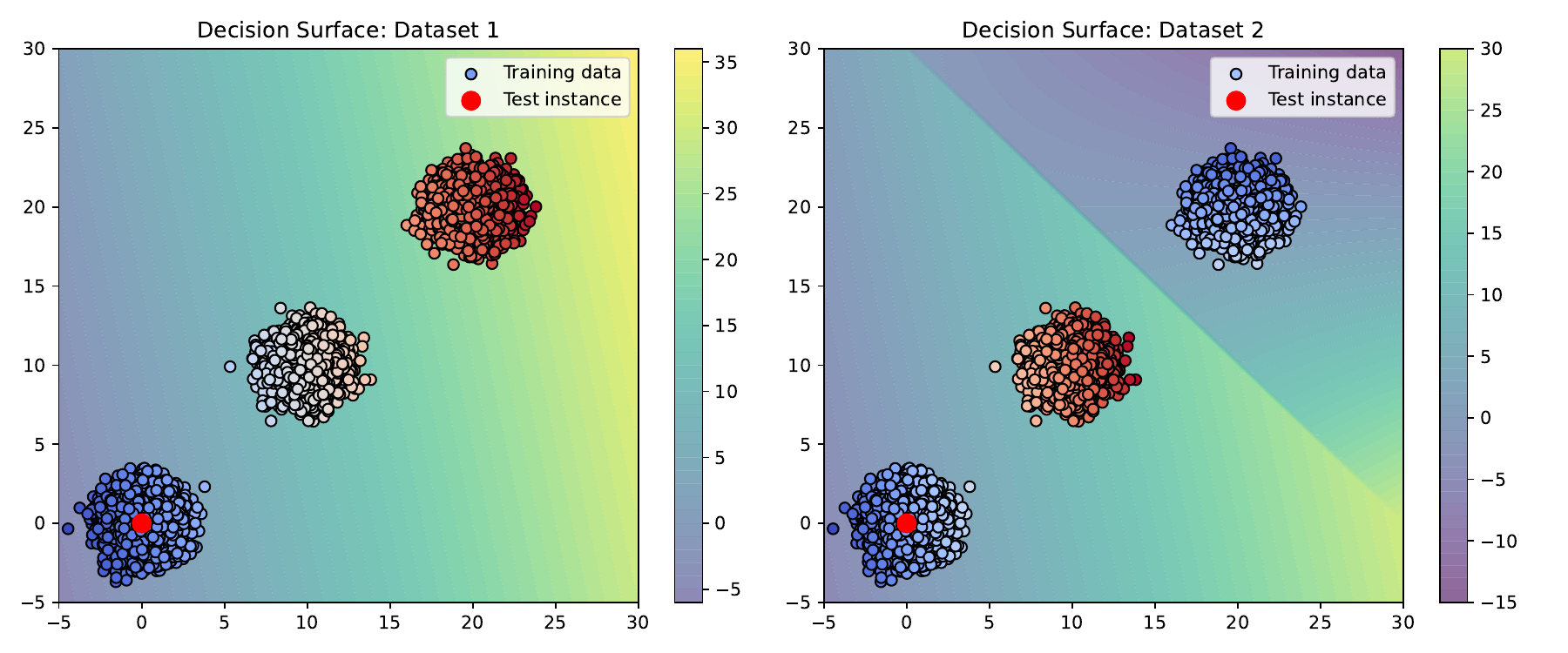}
\caption{Decision surface visualization of GTMs trained on Dataset 1 (left) and Dataset 2 (right). The test instance at $[0,0]$ is marked in red.}
\label{fig:decision.surfaces.ground.truth}
\end{figure}

\subsection{Dataset licences}
\label{appendix:licences}
All datasets used in this study (Wine Quality~\citep{cortez2009modeling}, Digits~\citep{pedregosa2011scikit}, and the Parkinson's Telemonitoring~\citealp{parkinsons_telemonitoring_189} dataset from UCI Machine Learning Repository) are licensed under a Creative Commons Attribution 4.0 International (CC BY 4.0) license.

\subsection{Compute resources}
\label{appendix:resources}
The computations were carried out on a FormatServer THOR E221 (Supermicro) server equipped with two AMD EPYC 7702 64-Core processors and 512 GB of RAM with operating system Ubuntu 22.04.1 LTS. All experiments were run using only 4 cores and 16GB of RAM. The experiments from case study section took less than 1 CPU hour Model training was computationally efficient, requiring less than 4 CPU hours in total. However, the comprehensive ranking calculation across all metrics, models, and datasets was more resource-intensive, consuming approximately 1000 CPU hours. The finetuning of metrics hyperparameters took around 200 CPU hours.

\end{document}